\documentclass{article}
\usepackage{microtype}
\usepackage{graphicx}
\usepackage{subfigure}
\usepackage{booktabs} % for professional tables
\usepackage{amsmath}
\usepackage{tikz-cd}
\usepackage{amsthm}
\usepackage{amssymb}
\usepackage{mathrsfs}
\usepackage{bm,bbm}
\usepackage{caption}
\usepackage{natbib}
\usepackage{enumitem}

\usepackage{hyperref}

\newtheorem{theorem}{Theorem}[section]

\newtheorem{proposition}[theorem]{Proposition}
\newtheorem{remark}{Remark}[section]

\newtheorem{assumption}{Assumption}[section]

\definecolor{darkgreen}{rgb}{0,0.6,0}
\definecolor{darkred}{rgb}{0.7,0.0,0}
\definecolor{darkblue}{rgb}{0,0.0,0.6}

\newcommand{\jcom}[1]{\textcolor{darkgreen}{}}
\newcommand{\xcom}[1]{\textcolor{blue}{}}

\newcommand{\jp}[1]{\textcolor{darkred}{}}
\newcommand{\sss}[1]{\textcolor{orange}{}}

\setcitestyle{authoryear,open={(},close={)}}

\usepackage[accepted]{icml2018}

\icmltitlerunning{Trade Selection with Supervised Learning and OCA}

\begin{document}

\twocolumn[
\icmltitle
{Trade Selection with Supervised Learning and OCA}




\begin{icmlauthorlist}
\icmlauthor{David Saltiel}{aisquare,lisic}
\icmlauthor{Eric Benhamou}{aisquare,lamsade}
\end{icmlauthorlist}

\icmlaffiliation{aisquare}{A.I. SQUARE CONNECT, 35 Boulevard d'Inkermann, 92200 Neuilly sur Seine, France.}
\icmlaffiliation{lisic}{LISIC - Universite du Littoral - Cote d’Opale, France.}
\icmlaffiliation{lamsade}{LAMSADE, Universite Paris Dauphine, 75016 Paris, France}

\icmlcorrespondingauthor{David Saltiel}{david.saltiel@aisquareconnect.com}

\vskip 0.3in
]


\printAffiliationsAndNotice{}  % leave blank if no need to mention equal contribution

\begin{abstract}
In recent years, state-of-the-art methods for supervised learning have exploited increasingly gradient boosting techniques, with mainstream efficient implementations such as xgboost or lightgbm. One of the key points in generating proficient methods is Feature Selection (FS). It consists in selecting the right valuable effective features. When facing hundreds of these features, it becomes critical to select best features. While filter and wrappers methods have come to some maturity, embedded methods are truly necessary to find the best features set as they are hybrid methods combining features filtering and wrapping. In this work, we tackle the problem of finding through machine learning best a priori trades from an algorithmic strategy. We derive this new method using coordinate ascent optimization and using block variables. We compare our method to Recursive Feature Elimination (RFE) and Binary Coordinate Ascent (BCA). We show on a real life example the capacity of this method to select good trades a priori. Not only this method outperforms the initial trading strategy as it avoids taking loosing trades, it also surpasses other method, having the smallest feature set and the highest score at the same time. The interest of this method goes beyond this simple trade classification problem as it is a very general method to determine the optimal feature set using some information about features relationship as well as using coordinate ascent optimization.
\end{abstract}

%\begin{keyword}
%feature selection, coordinate ascent, gradient boosting method
%\MSC[2010] 68T01, 68T05
%\end{keyword}

\section{Introduction: a motivating example}
\label{intro}
In financial markets, algorithmic trading has become more and more standard over the last few years. The rise of the machine has been particularly significant in liquid and electronic markets such as foreign exchange and futures markets reaching between 60 to 80 percent of total traded volume (see for instance \cite{Chan_2013}, \cite{Goldstein_2014} or \cite{Chaboud_2015} for more details on the various markets). These strategies are even more concentrated whenever there are very fast market moves as reported in \cite{Kirilenko_2017}.
 These algorithmic trading strategies typically relies on historical statistics. The main concept is to find some trading signals and information that identifies pattern or trend with a high probability of repetition. As desirable as it may be, the perfect algorithm is the one with the highest accuracy in terms of identifying the targeted pattern and with the smallest number of losing trades. 

If we want to increase robustness and bring additional firewalls to the trading strategy, it makes senses to add supplementary logic with the use of supervised learning method. The question is to empirically validate whether a supervised machine learning method can a priori identify bad or good trade and hence select among the systematic trades spawned by our algorithmic trading strategy. This is a typical supervised learning classification problem, very similar to the boiler plate example of identifying spam in emails. The complexity in this challenge is to identify features that are relevant to assist the machine in being able to in advance determine the chance of success of a machine based trade. 

This motivates for efficient method to select among a large set of features the ones that creates an efficient algorithm. This is precisely the subject of this paper. It is organized as follows. We first present the supervised learning classification problem. We then present the Optimal Cordinate Ascent algorithm that enables us selecting the Pareto optimal features set. The key contribution of this method is to exploit similarities between features and hence reduce the optimization search within categories as well as use coordinate ascent to transform the NP hard problem into a polynomial one. We then present results on a real life trading policy. We show that there is substantial improvement compared to the original strategy. We conclude on further work.

\section{Experience description}
\subsection{Challenge description}
A trading strategy is usually defined with some signal that generates a trading entry. But once we are in position, then next question is the trading exit strategy. There are multiple method to handle efficient exits, ranging from fixed target and stop loss, to dynamic target and stop loss.
Indeed, to enforce success and crystallize gain or limit loss, a common practice is to associate to the strategy a profit target and stop loss as described in various papers (\cite{Mauricio_2010}, \cite{Graziano_2014}, \cite{Stanley_2017}, or \cite{Vezeris_2018}).
 The profit target ensures that the strategy locks in real money the profit realized and is materialized by a limit order. The stop loss that is physically generated by a stop order safeguards the overall risk by limiting losses whenever the market backfires and contradicts the presumed pattern. To keep things simple we will hereby examine a trading strategy that has fixed profit target and stop loss. It generates about 1500 trades over a period of 10 years. For each of these trades, we make some measurements to get 135 features. 
The challenge is from these features to predict which trade is going to be successful. If we give brutally these features to a gradient boosting method like xgboost or lightgbm, the algorithm performs poorly as it is swamped by too many data that are noisy. The features that are provided are proprietary indicators whose identity and source are ignored by our machine learning algorithm. The challenge here is to find the optimal features set for our gradient boosting method. The learning process is summarized by figure \ref{learning}.

\begin{figure}
       \includegraphics[width=8cm]{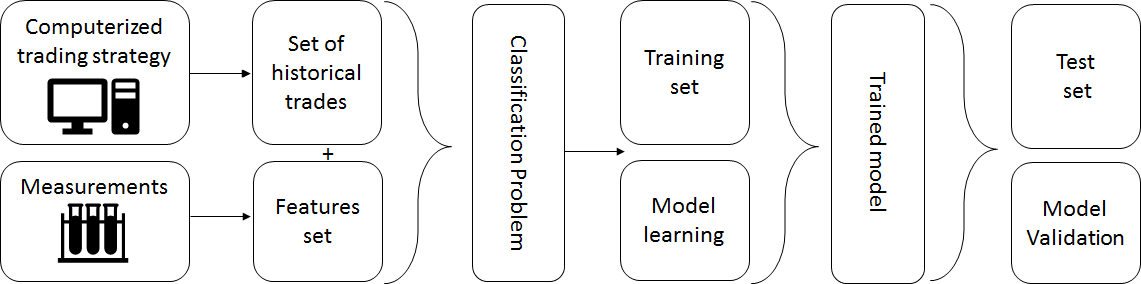}
	\caption{Learnig process for our trade selection challenge. We first use a proprietary trading strategy that generates some samples trades. We take various measures before the trades is executed to create a feature set. We combined these to create a supervised learning classification problem. Using xgboost method and OCA, we learn model parameters on a train set. We monitor overall performance of the trading strategy on a separate test set to validate scarce overfitting.}\label{learning}
	\centering
\end{figure}

\subsection{Feature selection}
Feature selection is also known as variable or attribute selection. It is the selection of a subset of relevant attributes in our data that are most relevant to our predictive modeling problem. 
It has been an active and fruitful field of research and development for decades in statistical learning.
It has proven to be effective and useful in both theory and practice for many reasons: 
enhanced learning efficiency and increasing predictive accuracy (see \cite{Mitra_2002}), model simplification to ease its interpretation and improve performance (see \cite{Almuallim_1994}, \cite{Koller_1996} and \cite{Blum_1997}), shorter training time (see \cite{Mitra_2002}), curse of dimensionality avoidance, enhanced generalization with reduced overfitting, implied variance reduction. Both \cite{HastieEtAl_2009} and \cite{Guyon_2003} are nice references to get an overview of various methods to tackle features selections.  The approaches followed varies. Briefly speaking, the methods can be sorted into three main categories: Filter method, Wrapper methods and Embedded methods. 

However, these methods do not exploit some particularities of our features set. We are able to regroup features among families. We call these features block variables. Typical example is to regroup variables that are observations of some physical quantity but at a different time (like the speed of the wind measure at different hours for some energy prediction problem, like the price of a stock in an algorithmic trading strategy for financial markets, like the temperature or heart beat of a patient at different time, etc ...).

\section{OCA Method}
The approach adopted here is the method referred to as the Optimal Cordinate Ascent (OCA) method that is described in \cite{SaltielBenhamou_2018}. Formally, we can regroup our variables into two sets:
\begin{itemize}[label={\tiny\raisebox{1ex}{\textbullet}}]
    \item the first set encompasses $B_1 \ldots B_n$. These are called block variables of different length $L_i$. Mathematically, the Block variables are denoted by $B_i$ with $B_i$ taking value in $\mathbb{R}^{L_i}\  ,\forall i \in 1 \ldots n$
    \item the second set is denoted $S$ and is a block of $p$ single variables.
\end{itemize}

Graphically, our variables looks like that:
\tiny
\begin{center}
$
\begin{pmatrix}
	\overbrace{
	\begin{matrix}
	    B_{1,1} & \ldots  & B_{1,n} \\
	    \hline
	    \bullet   & \ldots  & \bullet  \\ 
	     \vdots &  & \vdots  \\ 
	    \bullet  & \ldots  & \bullet 
	\end{matrix}
	}^{B_1}
	&
	\ldots \ldots
	&
	\overbrace{
	\begin{matrix}
	    B_{1,1} & \ldots  & B_{1,n} \\
	    \hline
	    \bullet   & \ldots  & \bullet  \\ 
	     \vdots &  & \vdots  \\ 
	    \bullet  & \ldots  & \bullet 
	\end{matrix}
	}^{B_n}
	&
	\overbrace{
	\begin{matrix}
	    S_{1} & \ldots  & S_{p} \\
	    \hline
	    \bullet  & \ldots  & \bullet  \\ 
	    \vdots &  & \vdots \\ 
	    \bullet  & \ldots  & \bullet 
	\end{matrix}
	}^{\text{S}}
\end{pmatrix}
$
\end{center}
\normalsize

In addition, we have $N$ variables split between block variables and single variables, hence $N = N_B +p$ with $N_B = \sum_{i=1}^n L_i$.

\subsection{Algorithm description}
Our algorithm works as follows. We first fit our classification model to find a ranking of features importance. The performance is computed with the Gini index for each variable. We then keep the first $k$ best ranked features for each blocks $B_1 \ldots B_n$ in order to find the best initial guess for our coordinate ascent algorithm. Notice that the set of unique variables is not modified during the first step of the procedure. The objective function is the number of  correctly classified samples at each iteration.
We then enter the main loop of the algorithm. Starting with the vector of $\left(k,\ldots ,k, \ \mathbbm{1}_p ^T \right) $ 
as the initial guess for our algorithm, we perform our coordinate ascent optimization in order to find the set with optimal score and the minimum number of features. The coordinate ascent loop stops whenever we either reach the maximum number of iterations or the current optimal solution has not moved between two steps.

\begin{algorithm}
\caption{OCA algorithm}\label{OCA}
\begin{algorithmic}
\STATE \textbf{J Best optimization}
\STATE We retrieve features importance from a fitted model
\STATE We find the index $k^\star$ that gives the best score for variables block of same size $k$: 
\STATE $k^\star \in \underset{k \in \mathbb{R}^{L_{\text{min}}}}{\text{argmax }} \text{Score}\left( k,\  \ldots ,\  k,\ \mathbbm{1}_p \right)$  
\COMMENT{$L_{\text{min}} =\underset{i \in \mathbb{R}^n}{\text{min }}L_i $}
\STATE Initial guess : $x^0 = \left(k^\star, \ldots, k^{\star}, \mathbbm{1}_p  \right)$ 

\WHILE{ $|\text{Score}( x^i )  \kern-0.5em - \kern-0.5em  \text{Score}( x^{i-1} ) |   \geq   \varepsilon_1  $ and    $ i   \leq \text{Iter max}_1 $}
	\STATE $ x_1^{i} \in \underset{j \in \mathbb{R}^{L_1}}{\text{argmax}} \text{ Score}\left( j,\  x_2^{i-1},\  x_3^{i-1} ,\  \ldots\ ,\  x_n^{i-1}, \mathbbm{1}_p\right)$
	\STATE ...
	\STATE $ x_n^{i} \in \underset{j \in \mathbb{R}^{L_n}}{\text{argmax}} \text{ Score}\left( x_1^{i},\  x_2^{i},\  x_3^{i} ,\ \ldots\ , j, \mathbbm{1}_p  \right)$
	\STATE i += 1
\ENDWHILE
\STATE
\STATE \textbf{Full coordinate ascent optimization}
\STATE Use previous solutions: $X^* = ( x_1^i, \ldots, x_n^{i},  \mathbbm{1}_p)$ \COMMENT{i is the last index in previous while loop}
\STATE $Y^* = \text{Score}\left(X^*\right)$ 

\WHILE{ $|Y-Y^*| \geq \varepsilon_2$ and iteration $\leq \text{Iter max}_2$ }
	\FOR{i=1 \ldots N}
		\STATE $ X = X^* $
		\STATE $ X_i = $not$\left(X_i^*\right) $			\COMMENT{not(0) = 1 and not(1) = 0}
		\IF{$ \text{Score}\left(X\right) \geq  \text{Score}\left(X^*\right)$} 
			\STATE $ X^* = X $
		\ENDIF
	\ENDFOR
	\STATE $ Y = \text{Score}\left(X^*\right) $
	\STATE iteration += 1
\ENDWHILE
\STATE Return $X^*, Y^*$
\end{algorithmic}
\end{algorithm}

We summarize the algorithm in the pseudo code \ref{OCA}. We denote by $\varepsilon$ the tolerance for the convergence stopping condition. To control early stop, we use a precision variable denoted by $\varepsilon_1, \varepsilon_2$ and two iteration maximum $\text{Iteration max}_1$ and $\text{Iteration max}_2$ that are initialized before starting the algorithm. We also denote $\text{Score}( k_1,\  \ldots ,\  k_n,\ \mathbbm{1}_p)$ to be the accuracy score of our classifier with each $B_i$ block of variables retaining $k_i$ best variables and with single variable all retained.

\begin{remark}
The originality of this coordinate ascent optimization is to regroup variable by block, hence it reduces the number of iterations compared to Binary Coordinate Ascent (BCA) as presented in \cite{ZARSHENAS_2016}
The stopping condition can be changed to accommodate for other stopping conditions.
\end{remark}
\begin{remark}
The specificity of our method is to keep the $j$ best representative features for each feature class, as opposed to other methods that only select one representative feature from each group, ignoring the strong similarities between each feature of a given variable block. This takes in particular the opposite view of feature Selection with Ensembles, Artificial Variables, and Redundancy Elimination as developed in \cite{Tuv_2009}.
\end{remark}

\section{Theoretical convergence speed}
Although it may be hard to determine the convergence speed for a real life example, under some weak conditions, we can prove that the convergence speed in linear. Hence we changed dramatically the nature of the problem as this method converts an NP hard problem into a polynomial one, making it feasible in a couple of minutes to train our model.

To formalize the concept, let us assume we examine the following optimization program: $\underset{x}{\min} f(x)$. We denote by $e_i$ the traditional vector with $0$ for any coordinate except $1$  for coordinate $i$. It is the vector of the canonical basis.

\begin{assumption} \label{assumption_convergence_prop}
We assume our function $f$ is twice differentiable and strongly convex with respect to the Euclidean norm:
\begin{equation}
f(y) \geq f(x) + \nabla f(x)^T (y-x) + \frac{\sigma}{2} \left\| y-x \right\|_2^2 
\label{strong_convexity}
\end{equation}
for some $\sigma > 0$ and any $x,y \in \mathbb{R}^n$. We also assume that each gradient's coordinate is uniformly $L_i$ Lipschitz, that is, there exists a constant $L_i$ such that for any $x \in \mathbb{R}^n, t \in \mathbb{R}$
\begin{equation}
\label{lipschitz}
\left| \left[\nabla f(x+t e_i)\right]_i - \left[\nabla f(x) \right]_i \right| \leq L_i \left| t \right|
\end{equation}
We denote by  $L_{\max}$ the maximum of these Lipschitz coefficients :
\begin{equation}
\label{Lmax}
L_{\text{max}} = \underset{i = 1 \ldots n}{max} \ L_i
\end{equation}

We assume that the minimum of $f$ denoted by $f^\star$ is attainable and that the left value of the epigraph with respect to our initial starting point $x_0$ is bounded, that is
\begin{equation}
\label{R0}
\underset{x}{\max} \left\{ \left\|  x - x^\star \right\| : f(x) \leq f(x_0) \right\}  \leq R_0
\end{equation}
\end{assumption}

\begin{remark}
Strong convexity means that the function is between two parabolas. Condition \ref{lipschitz} implies that the Gradient's growth is at most linear. Inequality \ref{R0} States that the function is increasing at infinity.
\end{remark}

\begin{proposition} \label{proposition1}
Under assumption \ref{assumption_convergence_prop}, coordinate ascent optimization (cf. Algorithm \ref{OCA}) converges to the global minimum $f^*$ at a linear rate proportional to $2n L_{\max} R_0^2$, that is
\begin{equation}
\label{theorem1}
\mathbb{E}[ f\left( x_k \right) ] - f^\star \leq \frac{2nL_{\text{max}}R_0^2}{k}
\end{equation}
\end{proposition}
\begin{proof}
See \cite{SaltielBenhamou_2018} appendix A.1 first part of the proof.
\end{proof}

\begin{proposition} \label{proposition2}
Under the same condition as proposition \ref{proposition1} and with $\sigma > 0$, we have an other convergence rate that decreases exponentially fast as follows:
\begin{equation}
\label{theorem2}
\mathbb{E}[ f\left( x_k \right) ] - f^\star \leq \left( 1 - \frac{ \sigma}{n L_{\max}}\right)^k (f(x_0)-f^\star)
\end{equation}
\end{proposition}

\begin{proof}
See \cite{SaltielBenhamou_2018} appendix A.1 second part of the proof.
\end{proof}

\begin{remark}
in the case of a large $\sigma$, the second rate of convergence is much faster than the first one.
\end{remark}

\begin{remark}
Our function to be maximize is obviously not convex. However, a linear rate in the convex case is rather a good performance for the ascent optimization method. Provided the method generalizes which is still under research, this convergence rate is a good hint of the efficiency of this method.
\end{remark}

\section{Numerical results}
We present herein the result of the machine learning experiment with a real life trading strategy. For full reproducibility, full data set and corresponding python code for this algorithm is available publicly on \href{https://github.com/davidsaltiel/OCA.git}{github} with the limitation that sensitive data have been either anonymized or removed (like for instance the final pnl curve).

We first compare our method with two other states of the art methods: Recursive Feature Elimination (RFE)  and Binary Coordinate Ascent (BCA) as presented in \cite{ZARSHENAS_2016}. 

Recursive Feature Elimination (RFE) (as presented in \cite{Mangal_2018}) first fits a model and removes features  until a pre-determined number of features. Features are ranked through an external model that assigns weights to each features and RFE recursively eliminates features with the least weight at each iteration. One of the main limitation to RFE is that it requires the number of features to keep. This is hard to guess a priori and one may need to iterate much more than the desired number of feature to find an optimal feature set.

Binary Coordinate Ascent (BCA) is an iterative deterministic local optimization method to find Feature subset selection (FSS). The algorithm searches throughout the space of binary coded input variables by iteratively optimizing the objective function in each dimension at a time. Because there is no similarities used in the coordinate ascent optimization, it performs slowly compared to OCA method.

On our test sets, we examine the accuracy score (the percentage of good classification). OCA method achieves the Pareto optimality as it reaches a score of 62.80  \%  with 16\% of features used, to be compared to RFE that achieves 62.80 \%  with 19\% of features used. BCA performs poorly with its highest score given by 62.19 \% with 27\% of features used. If we take in terms of efficiency criterium, the highest score with the less feature, OCA method is the most efficient among these three methods. In comparison, with the same number of features, namely 16\%, RFE gets a score of 62.40 \%. All these figures are summarized in the table \ref{tab:comparison}.

\begin{table}
	\centering
	\resizebox{8.3cm}{!}{
	\begin{tabular}{|c|c|c|c|c|}
		\hline
		Method & OCA & RFE 24 features  & BCA & RFE 28 features\\
		\hline
		\% of features &\textcolor[rgb]{1,0, 0}{16.6} & \textcolor[rgb]{ 1,0,0}{16.6} & \textcolor[rgb]{ 0,0,0.8}{27.08}&\textcolor[rgb]{1,0.4, 0}{19.4}\\
		\hline
		Score (in \%) &\textcolor[rgb]{1, 0, 0}{62.8} & \textcolor[rgb]{1,0.4, 0}{62.39} & \textcolor[rgb]{0,0,0.8}{62.19}&\textcolor[rgb]{1,0,0}{62.8}\\
		\hline
	\end{tabular}
	}
	\caption{Method Comparison: for each row, we provide in red the best(s) (hotest) method(s) and in blue the worst (coldest) method, while intermediate methods are in orange. We can notice that OCA achieves the higher score with the minimum feature sets. For the same feature set, RFE performs worst or equally, if we want the same performance for RFE, we need to have a larger feature set. BCA is the worst method both in terms of score and minimum feature set.}
	\label{tab:comparison}%
\end{table}%

It is illuminating to look at the histogram of gain and losses of our trades over our 10 years of history. Not surprisingly, we can observed two peaks corresponding to the profit target and stop loss level as shown in figure  \ref{graph1}. This is quite obvious, but it is much better to use the pnl curve in the native currency of the underlying instrument than to look at the consolidated currency of our trading strategies to avoid foreign exchange noise as shown in figure \ref{graph2}.

\begin{figure}
	\centering
       \includegraphics[scale=0.4]{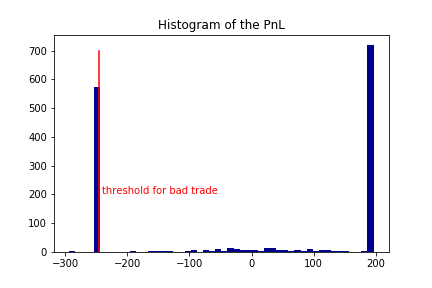}
	\caption{Histogram of the PnL in Dollars amount (re-normalized for anonymity). We can observe two peaks corresponding to the profit target and stop loss levels. This is logical as the trading strategy examined here is a fixed profit target and stop loss strategy. As soon as a trade reaches these levels, the gain or loss is crystallized. If the market stays in trading range and do not reach the level, we have a timeout in the strategy that cut the strategy regardless of its pnl. These cases are rather rare and hence represents very small bars in the histogram.}\label{graph1}
\end{figure}

\begin{figure}
	\centering
       \includegraphics[scale=0.4]{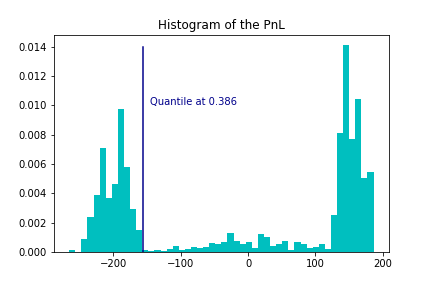}
	\caption{Same Histogram of the PnL but in Euros. Although it may seem very basic, it is important to use the native currency of the algorithmic trading strategy to avoid currency noise. Compared to figure \ref{graph1}, the only difference is to observe the profit and loss not in dollars but in euros as we consolidate all our trading strategies in euros. This is not a good practice as it introduces some additional noise in our labels as the Eur Dollar fx rate randomises slightly the pnl outcome and hence some time out exit may be confused with some bad exits.}\label{graph2}
\end{figure}

\section{Discussion}
Compared to BCA our method reduces the number of iterations as it uses the fact that variables can be regrouped into categories or classes. Below is provided the number of iterations for OCA and BCA in figure \ref{convergence}. Our method requires only 350 iterations steps ton converge as opposed to BCA that needs up to 700 iterations steps as it computes blindly variables ignoring similarities between the different variables.

\begin{figure}
	\centering
       \includegraphics[width = 4cm]{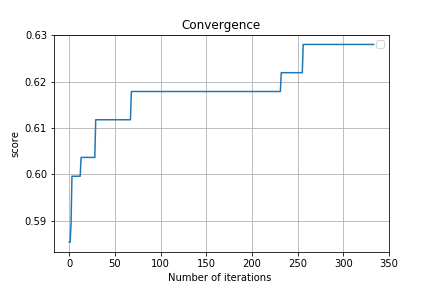} 	 \includegraphics[width = 4cm]{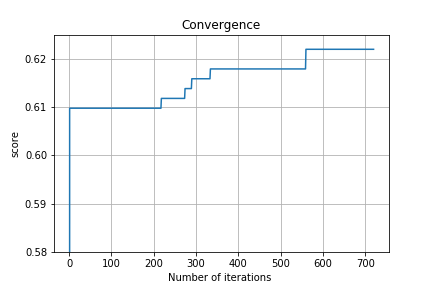}
	\caption{Iterations steps up to convergence for OCA and BCA. OCA method is on the left while BCA is on the right. We see that OCA requires around 350 iteration steps to converge while BCA requires the double around 700 iteration steps to converge}\label{convergence}
\end{figure}

Graphically, we can compute the best candidates for the four methods listed in table \ref{tab:comparison} in figure \ref{fig:comparison} and \ref{fig:comparison2}. We have taken the following color code. The hottest (or best performing) method is plotted in red, while the worst in blue. Average performing methods are plotted in orange. In order to compare finely OCA and RFE, we have plotted in figure \ref{fig:comparison2} the result of RFE for used features set percentage from 10 to 30 percent. We can notice that for the same feature set as OCA, RFE has a lower score and equally that to get the same score as OCA, RFE needs a large features set.

\begin{figure}
	\centering
       \includegraphics[width = 6 cm]{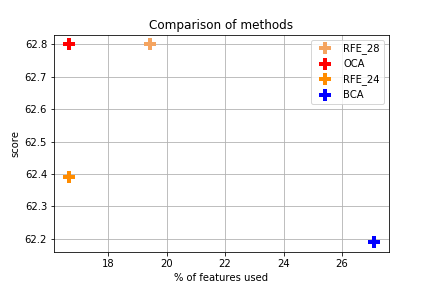}
	\caption{Comparison between the 4 methods. To qualify the best method, it should be in the upper left corner. The desirable feature is to have as little features as possible and the highest score. We can see that the red cross that represents OCA is the best. The color code has been designed to ease readability. Red is the best, orange is a slightly lower performance while blue is the worst.} 	\label{fig:comparison}
\end{figure}

\begin{figure}
	\centering
       \includegraphics[width = 6 cm]{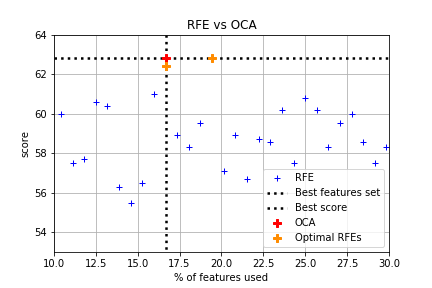}
	\caption{Comparison between OCA and RFE. Zoom on the methods. For RFE, we provide the score for various features set in blue. The two best RFE performers points are the orange cross marker points that are precisely the one listed in table \ref{tab:comparison}. The red cross marker point represents OCA. It achieves the best efficiency as it has the highest score and the smallest feature set for this score.} \label{fig:comparison2}
\end{figure}

We then look at the final goal which is to compare the trading strategy with and without machine learning. A standard way in machine learning is to split our data set between a randomized training and test set. We keep one third of our data for testing to spot any potential overfitting. If we use the standard and somehow naive way to take randomly one third of the data for our test set, we break the time dependency of our data. This has two consequences. We use in our training set some data that are after our test sets which is not realistic compared to real life. We also neglect any regime change in our data by mixing data that are not from the same period of time. However, we can do the test on this mainstream approach and compare the trading strategy with and without machine learning filtering. This is provided in figure \ref{experience1}. Since the blue curve that represents the combination of our algorithmic trading strategy and the oca method is above the orange one, we experimentally validate that using machine learning enhances the overall profitability of our trading strategy by avoiding the bad trades.

\begin{figure}[h]
	\centering
       \includegraphics[width =8cm]{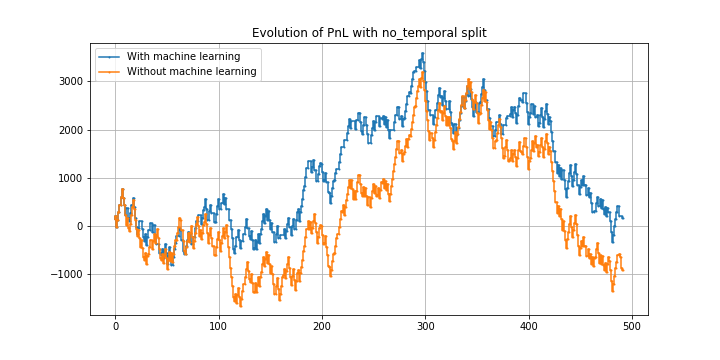}
	\caption{Evolution of the PnL with a randomized test set. The orange curve represents our algorithmic trading strategy without any machine learning filtering while the blue line is the result of the combination of our algorithmic trading strategy and the oca method to train our xgboost method} \label{experience1}
\end{figure}

If instead we split our set into two sets that are continuous in time, meaning we use as a training test the first two third of the data when there are sorted in time and as a test set the last third of the data, we get better result as the divergence between the blue and orange curve is larger. An explanation of this better efficiency may come from the fact that the non randomization of the training set makes the learning for our model easier and leads to less overfitting overall. The method of splitting the two sets: training and test set into two sets relies on a temporal split, hence the title of our figure \ref{experience2}.

\begin{figure}[h]
	\centering
       \includegraphics[width =8cm]{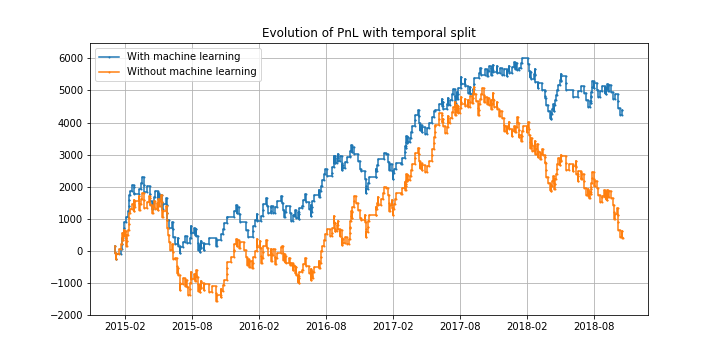}
	\caption{Evolution of the PnL with a test set given by the last third of the data to take into account temporality in our data set. The orange curve represents our algorithmic trading strategy without any machine learning filtering while the blue line is the result of the combination of our algorithmic trading strategy and the OCA method to train our xgboost method}\label{experience2}
\end{figure}

Last but not least, we can zoom the two curves when taking the test set with a temporal split. We clearly see that the method performs well to avoid selecting bad trades and hence the blue line decreases less than the orange one as shown in figure \ref{experience3}.

\begin{figure}[h]
	\centering
       \includegraphics[width =8cm]{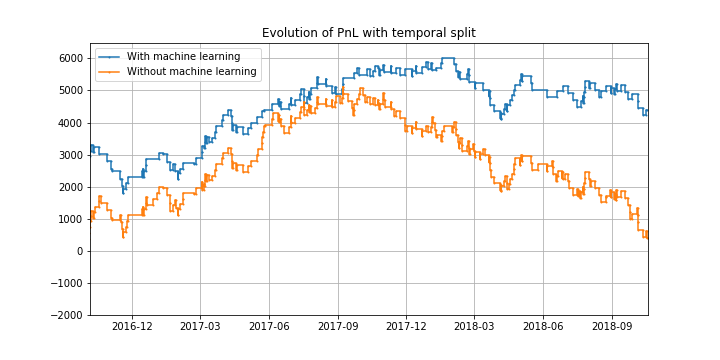}
	\caption{Zoom of the evolution of the PnL with a temporal split.The orange curve represents our algorithmic trading strategy without any machine learning filtering while the blue line is the result of the combination of our algorithmic trading strategy and the OCA method to train our xgboost method}\label{experience3}
\end{figure}

\section{Conclusion}
Algorithmic trading method can be enhanced with supervised learning method. The challenge is to use measurements and information regrouped into features to detect before orders are electronically sent to the exchange highly probable non successful trades. Because the logic of the algorithmic trading strategy may be challenging to understand, an agnostic supervised learning method can come to the rescue. However, choosing the best features in our initial features set is tricky as more data simultaneously provide additional information and noise at the same time. We present here OCA, a new feature selection method that leverages similarities between features. This method is not very demanding in terms of features knowledge and can efficiently select best features without testing all possible features sets. This changes the features selection problem from an NP hard one into a polynomial one. When implemented on real case strategies, we can empirically validate that the supervised learning method enhances overall trading profitability. As we ask the algorithm to detect in pre-trade operations highly unsuccessful candidates, the method is logically able to reduce overall draw-downs. The method developed herein is quite general and can be applied to any general supervised learning binary classification. In further work, we would like to explore reinforcement learning method to adjust our method for capacity constraints as this is a limitation of the supervised learning approach.
\bibliography{mybibfile}

\end{document}